\documentclass{article}
\usepackage{ijcai22}

\usepackage{times}

\usepackage{soul}
\usepackage{url}
\usepackage[hidelinks]{hyperref}
\usepackage[utf8]{inputenc}
\usepackage[small]{caption}
\usepackage{graphicx}
\usepackage{amsmath}
\usepackage{booktabs}
\usepackage{paralist}
\usepackage{amsfonts}
\usepackage{comment}
\usepackage{amssymb}
\usepackage{amsthm}
\usepackage{dsfont}
\usepackage{mathtools}
\usepackage{subcaption}
\usepackage{xcolor}
\usepackage{algorithm}
\usepackage{algorithmic}
\usepackage{float}

\newcommand{\twodots}{\mathinner {\ldotp \ldotp}}

\DeclareMathOperator*{\argmax}{arg\,max}

\newtheorem{theorem}{Theorem}[]
\newtheorem{example}{Example}[]
\newtheorem{proposition}[theorem]{Proposition}
\newtheorem{definition}{Definition}[]

\newtheorem*{remark}{Remark}

\newcommand\Voters{\mathcal{N}}

\newcommand\cj[1]{}
\newcommand\ct[1]{}

\definecolor{burgundy}{rgb}{0.8, 0.0, 0.13}
\def\bt{\color{black}}
\def\et{\color{black}}

\title{Multi-winner Approval Voting Goes Epistemic}

\author{
Tahar Allouche$^1$\footnote{Contact Author}\and
Jérome Lang$^1$\And
Florian Yger$^{1}$\\
\affiliations
$^1$LAMSADE, CNRS, PSL, Université Paris-Dauphine\\
\emails
tahar.allouche@dauphine.eu, lang@lamsade.dauphine.fr, florian.yger@lamsade.dauphine.fr
}

\begin{document}

\maketitle

\begin{abstract}
Epistemic voting interprets votes as noisy signals about a ground truth. We consider contexts where the truth consists of a set of objective winners, knowing a lower and upper bound on its cardinality. 
A prototypical problem for this setting is the aggregation of
multi-label annotations with prior knowledge on the size of the ground truth.
We posit 
noise models, for which we define rules that output an optimal set of winners.
We report on experiments on 
multi-label annotations (which we collected).
\end{abstract}

\section{Introduction}

The epistemic view of voting assumes the existence of a ground truth which, usually,
is either an alternative or a ranking over alternatives.  Votes reflect opinions or beliefs about this ground truth; the goal is to aggregate these votes so as to identify it. Usual methods define a noise model specifying the probability of each voting profile given the ground truth, and output the alternative that is the most likely state of the world, or the ranking that is most likely the true ranking. 

Now, there are contexts where the ground truth does not consist of a single alternative nor a ranking, but of a {\em set of alternatives}. Typical examples are multi-label crowdsourcing (find the items in a set that satisfy some property, {\em e.g.} the sport teams appearing on a picture)
or finding the  objectively $k$ best candidates 
(best papers at a conference, 
best performance in artistic sports, 
 $k$ patients with highest probabilities of survival if being assigned a scarce medical resource).\

These alternatives that are truly in the ground truth are called `winning' alternatives. Depending on the context, the number of winning alternatives can be fixed, unconstrained, or more generally, constrained to be in a given interval.
This constraint 
expresses some {\em prior knowledge on the cardinality of the ground truth}.
Here are some examples:
\begin{itemize}
    \item {\em Picture annotation via crowdsourcing}: participants are shown a picture taken from a soccer match and have to identify the team(s) appearing in it. The ground truth is known to contain one or two teams. 
    \item{\em Guitar chord transcription}: voters are base classifier algorithms \cite{aggregation2020} which, for a given chord, select the set of notes constitute it. The true set of notes can contain 
    three to 
    six alternatives. 
    \item {\em Jury}: participants are members of a jury which has to give an award to three papers presented at a conference: the number of objective winners is fixed to three.
 (In a variant, the number of awards would be {\em at most} three.)
    \item {\em Resource allocation}: participants are doctors and alternatives are Covid-19 patients in urgent need of intensive care; there is a limited number $k$ of intensive care units. The ground truth consists of those patients who most deserve to be cured (for example those with the $k$ highest probabilities of survival if cured). 
\end{itemize}

We assume that voters provide a simple form of information: {\em approval ballots}, indicating which alternatives they consider plausible winners. These approval ballots are not subject to any cardinality constraint: {\em a voter may approve a number of alternatives, even if it does not lie in the interval bearing on the output}. This is typically the case for totally ignorant voters, who are expected to approve all alternatives. 
Sometimes, the aggregating mechanism
has some prior information about the likelihood of alternatives and the reliability of voters. We first study a simple case where this information is specified in the input: in the noise model, each voter has a probability $p_i$ (resp. $q_i)$ of 
approving a winning (resp. non-winning) alternative,
and each alternative 
has a prior probability 
to be winning.
This departs from classical voting, where voters are usually treated equally ({\em anonymity}), and similarly for alternatives ({\em neutrality}). 


This simple case serves as a building component for the more complex case where these parameters are not known beforehand but {\em estimated from the votes}: votes allow to infer information about plausibly winning alternatives, from which we infer information about voter reliabilities, which leads to revise information about winning alternatives, 
and so on until the process converges. Here we move back to an anonymous and neutral setting, since all alternatives (resp. voters) are treated equally before votes are known.

After discussing related work (Section \ref{sec:related}), we introduce the model (Section \ref{sec: Prior}) 
and give an estimation algorithm (Section \ref{sec:estimating}), first in the case where the parameters are known, and then in the case where they are estimated from the votes.  
In Section \ref{sec: experiments} we present a data gathering task and analyse the results of the experiments.
Section \ref{conclusion} concludes.

\section{Related Work}\label{sec:related}

\paragraph{Epistemic social choice}
It 
studies how a ground truth can be recovered from noisy votes, viewing voting rules as maximum likelihood estimators. 
Condorcet's {\em jury theorem} \cite{condorcet1785}  considers $n$ independent, equally reliable voters and two alternatives that are {\em a priori} equally likely, and  states that if every voter votes for the correct alternative with probability $p>\frac{1}{2}$, then the majority rule outputs the correct decision with a probability that increases with $n$ and tends to 1 when $n$ grows to infinity. See \cite{collective2017} and \cite{Premises2008} for proofs and discussion.

The framework was later generalized to more than two alternatives 
\cite{condorcet1988}, to voters with different competences  \cite{ShapleyGrofman84,maximum2004}, to a nonuniform prior over alternatives \cite{Ben-YasharN97,optimal2001}, 
to various noise models \cite{common2005,ConitzerRX09}, to correlated votes \cite{voting2011,epistemic2017}, and to multi-issue domains \cite{XiaCL10}. 
\cite{truth2019} define a method to aggregate votes weighted according to their average proximity to the other votes as an estimation of their reliability. A review of the field can be found  in 
\cite{ElkindSlinko16}. 

Epistemic voting with approval ballots has scarcely been considered. \cite{isapproval2015} study noise models for which approval voting is optimal given $k$-approval votes, in the sense that the objectively best alternative gets elected, the ground truth being a ranking over all alternatives. \cite{AlloucheLY22} do a similar work for a ground truth consisting of a single candidate.
\cite{learning2017} prove that the number of samples needed to recover the ground truth ranking over alternatives with high enough probability from approval ballots is exponential if ballots are required to approve $k$ candidates, but polynomial if the size of the ballots is randomized.

\paragraph{Multi-winner voting rules}
They output a set of alternatives (of fixed cardinality or not) from a set of votes (approvals or rankings). There have been a lot of recent developments in the field (a recent survey is \cite{FaliszewskiSST17}), mostly concerning the classical (non-epistemic) view of social choice, where votes express preferences. 

Multi-winner epistemic voting has received only little attention.
\cite{maximum2012} assume a ground truth ranking over alternatives, and identify rules that output the $k$ alternatives maximizing the likelihood to contain the best alternative, or the likelihood to coincide with the top-$k$ alternatives. 
The last section of \cite{maximum2011} defines a noise model where the ground truth is a set of $k$ alternatives (and the reported votes are partial orders).
The only work we know where the noise models produce random {\em approval votes} from a ground truth consisting of {\em a set of alternatives} is \cite{Evaluating2020}.
They define a family of distance-based noise models, whose prototypical instance
generates approval votes selecting an alternative in the ground truth
(resp. not in the ground truth) with probability $p$ (resp. $1-p$);
as we see further, this is a specific case of our noise model.


\paragraph{Crowdsourcing}


\cite{Axiomatic2014,Empirical2014} give a social choice-theoretic study of collective annotation tasks. \cite{ShahZ20} design mechanisms for incentive-compatible elicitation with approval ballots in crowdsourcing applications. 
 Beyond social choice, collective multi-label annotation was first addressed in \cite{reliable2010}, which studies the agreement between experts and non-experts in some multi-labelling tasks, and in \cite{scalable2014}, where a scalable aggregation method is presented to solve the multi-label estimation problem.


\section{The Model}\label{sec: Prior}



Let $\Voters=\{1,\dots,n\}$ be a set of voters, and $\mathcal{A}=\{a_1,\dots,a_m\}$ a set of alternatives (possible objects in images, notes in chords, papers, patients...).
Consider a set of $L$ {\em instances}: an instance $z$ consists of an approval profile $A^z=(A_1^z,\dots,A_n^z)$ where $A_i^z \subseteq \mathcal{A}$ is an approval ballot for every $i \in \Voters$. For example, in a crowdsourcing context, a task usually contains multiple questions, and an instance comprises the voters' answers to one of these questions.

For each instance $z\in L$, there exists an \emph{unknown} ground truth $S^*_z$ belonging to $\mathcal{S}=2^{\mathcal{A}}$, which is the set of objectively correct alternatives in instance $z$. 
It is common knowledge that the number of alternatives in each of them lies in the interval $[l,u]$:
$S^*_z \in \mathcal{S}_{l,u}= \{S \in \mathcal{S}, l\leq |S| \leq u\}$, for given bounds $0\leq l \leq u\leq m $.

Our goal is to unveil the ground truth for each of these instance using the votes and the prior knowledge on the number of winning alternatives.
We 
define a noise model consisting of two parametric distributions, namely, a conditional distribution of the approval ballots given the ground truth, and a prior distribution on the ground truth. Here we depart from classical noise models in epistemic social choice, as we suppose that the parameters of these distributions may be unknown and thus need to be estimated.

For each voter $i\in \Voters$, we suppose that there exist two unknown parameters $(p_i,q_i)$ in $(0,1)$ such that the approval ballot $A_i^z$ on an instance $z \in L$ is drawn according to the following distribution: for each $a \in {\cal A}$,
$$
P(a \in A_i^z|S^*_z=S) = \left\{
    \begin{array}{ll}
        p_i & \mbox{if } a \in S \\
        q_i & \mbox{if } a \notin S
    \end{array}
\right. 
$$
 where 
$p_i$ (resp. $q_i$) is the (unknown) probability that voter $i$ approves a correct (resp. incorrect) alternative.
Then we make the following assumptions:
\begin{compactitem}
    \item[(1)]   A voter's  approvals of alternatives  are  mutually  independent given 
    the ground truth and parameters $(p_i,q_i)_{i \in \Voters}$.
   \item[(2)]   Voters' ballots are  mutually  independent given the ground truth.
    \item[(3)] Instances are independent given the parameters $(p_i,q_i)_{i \in \Voters}$ and the ground truths.
\end{compactitem}

To model the prior probability of any set $S$ to be the ground truth $S^*$, we define
parameters $t_j=P(a_j \in S^*)$. $t_j$ can be understood as the prior probability of $a_j$ to be in the ground truth set $S^*$  before the cardinality constraints are taken into account. These, together with an independence assumption on the events $\{a_j \in S^*\}$, gives 
$P(S=S^*)=\prod\limits_{a_j \in S} t_j \prod\limits_{a_j \notin S}1- t_j$.  Note that the choice of the parameters $t_j$ is not crucial when running the algorithm for estimating the ground truth: we will see in Section \ref{subsec:amle} that it converges whatever their values.  
The distribution conditional to the prior knowledge on the size of the ground truth can 
be seen as a projection on the constraints  followed by a normalization: 
\begin{small}
$$\tilde{P}(S)=P(S^*=S|l\leq|S^*|\leq u)=\frac{P(S^*=S \cap |S^*|\in [l,u])}{P(|S^*|\in [l,u])} $$
\end{small}
It follows:
$$\tilde{P}(S) = \left\{
    \begin{array}{ll}
        \frac{1}{\beta(l,u,t)}\prod\limits_{a_j\in S}t_j \prod\limits_{a_j \notin S} (1-t_j) & \mbox{if } S \in \mathcal{S}_{l,u} \\
        0 & \mbox{if } S \notin \mathcal{S}_{l,u}  
    \end{array}
\right. $$
where
$\beta(l,u,t)=\sum\limits_{S\in \mathcal{S}_{l,u}} \prod\limits_{a_j\in S}t_j \prod\limits_{a_j \notin S} (1-t_j)$.

The ground truths associated with different instances are assumed to be mutually independent given the parameters.

Two particular cases are worth discussing. First, when $(l,u)=(0,m)$, the problem is {\em unconstrained} and we have $\beta(0,m,t)=P(|S^*| \in [0,m])=1$, so $\tilde{P}(S)=P(S=S^*)$.
In this case the problem degenerates into a series of independent binary label-wise estimations (see Subsection \ref{subsec:gt-vp}). 
  
Second, in the single-winner case $(l,u)=(1,1)$, we have $\tilde{P}(\{a_j\})
=\frac{t_j \prod_{h\neq j}1-t_h}{\beta(1,1,t)}$, therefore, for any approval profile $A$,  $P(S^* = \{a_j\}|A,|S^*|=1)\propto \frac{t_j}{1-t_j} P(A|S^* = \{a_j\})$. We recover the same estimation problem if we simply introduce $\alpha_j=P(S^* = \{a_j\})$ with $\sum \alpha_j =1$ as in \cite{optimal2001},
in which case we have $P(S^* = \{a_j\}|A,|S^*|=1)\propto \alpha_j P(A|S^* = \{a_j\})$.

\section{Estimating the Ground Truth}
\label{sec:estimating}

Our aim is the intertwined estimation of the ground truth and the parameters via maximizing the total likelihood of the instances:
\begin{small}
\begin{align*}
 \mathcal{L}(A,S,p,q,t) &=\prod_{z=1}^L \tilde{P}(S_z) \prod_{i=1}^n P(A_i^z|S_z)
 \end{align*}
\end{small} 
where:
\begin{small}
$$ P(A_i^z|S_z)=p_i^{|A_i^z\cap S_z|} q_i^{|A_i^z\cap \overline{S_z}|}(1-p_i)^{|\overline{A_i^z}\cap S_z|} (1-q_i)^{|\overline{A_i^z}\cap \overline{S_z}|} $$
\end{small}
To this aim, we will introduce an iterative algorithm whose main two steps will be presented in sequence, in the next subsections, before the main algorithm is formally defined and its convergence shown. These two steps are:
\begin{compactitem}
    \item Estimating the ground truths given the parameters.
    \item Estimating the parameters given the ground truths.
\end{compactitem}
Simply put, the algorithm consists in iterating these two steps until it converges to a fixed point. 

\subsection{Estimating the Ground Truth Given the Votes and the Parameters}\label{subsec:gt-vp}

Since instances are independent given the parameters, we focus here on one instance with ground truth $S^*$ and 
profile $A=(A_1,\dots,A_n)$. Before diving into 
maximum likelihood estimation (MLE), 
we introduce some notions and prove some lemmas. In this subsection, we suppose that the parameters $(p_i,q_i)_{i \in \Voters}$ and $(t_j)_{j \in \mathcal{A}}$ are known (later on, these parameters will be replaced by their estimations at each iteration of the algorithm).
Thus, all in all, input and output are as follows:
\begin{compactitem}
    \item Input: approval profile $A$;
    parameters $(p_i,q_i)_{i \in \Voters}$ and $(t_j)_{j \in \mathcal{A}}$.
    \item Output: MLE of the ground truth $S^*$.
\end{compactitem}

\begin{definition}[weighted approval score]\label{weighted approval score}
Given an approval profile $(A_1,\dots,A_n)$, noise parameters $(p_i,q_i)_{1\leq i \leq n}$ 
and prior parameters $(t_j)_{1\leq j \leq m}$, define:
$$app_w(a_j)=ln\left(\frac{t_j}{1-t_j}\right) + \sum_{i: a_j\in A_i} ln\left(\frac{p_i(1-q_i)}{q_i(1-p_i)}\right)$$
\end{definition}

The scores $app_w(a_j)$ can be interpreted as weighted approval scores for a $(n+m)$-voter profile where:
\begin{compactitem}
   \item for each voter $1\leq i \leq n$: $i$ has a weight $w_i=ln\left(\frac{p_i(1-q_i)}{q_i(1-p_i)}\right)$ and casts approval ballot $A_i$.
    \item for each $1\leq j \leq m$: there is a virtual voter with weight $w_j=ln\left(\frac{t_j}{1-t_j}\right)$ who casts approval ballot $A_j=\{a_j\}$.
\end{compactitem}
While the weight of each voter $i \in \Voters$ depends on her reliability, each prior information on an alternative plays the role of a virtual voter who only selects the concerned alternative, with a weight that increases as the prior parameter increases.

From now on, we suppose without loss of generality that the alternatives are ranked according to their score: 
$$app_w(a_1) \geq app_w(a_2) \geq \dots \geq app_w(a_m) $$

\begin{definition}[threshold and partition]\label{threshold_partition}
Define the threshold:
$$\tau_n=\sum_{i=1}^n ln\left( \frac{1-q_i}{1-p_i}\right)$$
and the partition of the set of alternatives in three sets:
$$
\left\{
    \begin{array}{ll}
        S_{max}^{\tau_n} & =\left\{a\in A, app_w(a)>\tau_n\right\} \\
        S_{tie}^{\tau_n} & =\left\{a\in A, app_w(a)=\tau_n \right\}\\
        S_{min}^{\tau_n} & =\mathcal{A}\backslash (S_{max}^{\tau_n}\cup S_{tie}^{\tau_n})
        \end{array}
\right.
$$
and let $k_{max}^{\tau_n}=|S_{max}^{\tau_n}|, k_{tie}^{\tau_n}=|S_{tie}^{\tau_n}|, k_{min}^{\tau_n}=|S_{min}^{\tau_n}|$.
\end{definition}

The next result characterizes the sets in $\mathcal{S}$ that are MLEs of the ground truth given the parameters. 
\begin{theorem}\label{constrained}
$\Tilde{S} \in \argmax_{S\in \mathcal{S}} \mathcal{L}(A,S,p,q,t)$ if and only if there exists $k\in [l,u]$ such that 
\bt$\Tilde{S}$ is the set of $k$ alternatives with  the highest $k$ values of $app_w$ and:\et
\begin{equation}\label{constraints}
\left\{
    \begin{array}{cl}
        |\Tilde{S}\cap S_{max}^{\tau_n}| & =\min(u,k_{max}^{\tau_n})\\
        |\Tilde{S}\cap S_{min}^{\tau_n}| & =\max(0,l-k_{tie}^{\tau_n}-k_{max}^{\tau_n})
        \end{array}
\right.
\end{equation}
\end{theorem}

So the estimator $\tilde{S}$ is made of some top-$k$ alternatives, where the possible values of $k$ are determined by Eq.~(\ref{constraints}). The first equation imposes that $\Tilde{S}$ includes as many elements as possible from $S_{max}^{\tau_n}$ (without exceeding the upper-bound $u$), whereas the second one imposes that $\Tilde{S}$ includes as few elements as possible from $S_{min}^{\tau_n}$ (without getting below the lower-bound $l$). An example is included in the appendix.

\begin{proof}
Since $\tilde{P}(S)>0 \iff S \in \mathcal{S}_{l,u}$, we have that $\argmax_{S\in \mathcal{S}} L(S)= \argmax_{S\in \mathcal{S}_{l,u}} L(S)$.
Moreover, we have that for any $S \in \mathcal{S}_{l,u}$:
\begin{small}{\allowdisplaybreaks
\begin{align*}
L(S)& = \tilde{P}(S)\prod_{i=1}^n p_i^{|A_i\cap S|}q_i^{|A_i\cap \overline{S}|}(1-p_i)^{|\overline{A_i}\cap S|}(1-q_i)^{|\overline{A_i}\cap \overline{S}|}\\
 & = \tilde{P}(S)\prod_{i=1}^n p_i^{|A_i\cap S|}q_i^{|A_i|-|A_i\cap S|}(1-p_i)^{|S|-|A_i\cap S|}\\
 & \mbox{~~~~~~~~~~~~~~~~~~~~~}(1-q_i)^{|\overline{A_i}|-|S|+|A_i\cap S|}\\
 & \propto \tilde{P}(S) \prod_{i=1}^n \left[\frac{1-p_i}{1-q_i}\right]^{|S|}\left[\frac{p_i(1-q_i)}{q_i(1-p_i)}\right]^{|A_i \cap S|}\\
 & \propto \frac{1}{\beta}\prod_{a_j \in S} t_j \prod_{a_j \notin S} (1-t_j) \prod_{i=1}^n \left[\frac{1-p_i}{1-q_i}\right]^{|S|}\left[\frac{p_i(1-q_i)}{q_i(1-p_i)}\right]^{|A_i \cap S|}\\
  & \propto \prod_{a_j \in S} \frac{t_j}{1-t_j}\prod_{i=1}^n \left[\frac{1-p_i}{1-q_i}\right]^{|S|}\left[\frac{p_i(1-q_i)}{q_i(1-p_i)}\right]^{|A_i \cap S|} 
  \end{align*}}
\end{small}

Thus the log-likelihood reads:

\allowdisplaybreaks
\begin{small}
\begin{align*}
    l(S) & = \sum_{a_j\in S} \ln \frac{t_j}{1-t_j}+\sum_{i=1}^n |S| \ln \frac{1-p_i}{1-q_i}+ |A_i \cap S|\ln \frac{p_i(1-q_i)}{q_i(1-p_i)}\\
     & = \sum_{a_j \in S} \left[ \overbrace{\underbrace{\ln \frac{t_j}{1-t_j}+\sum_{i: a_j \in A_i} \ln \frac{p_i(1-q_i)}{q_i(1-p_i)}}_{app_w(a_j)}-\underbrace{\sum_{i=1}^n\ln \frac{1-q_i}{1-p_i}}_{\tau_n} }^{l(a_j)}\right]
\end{align*}

\end{small}
This means that $a\in S_{max}^{\tau_n}$ if and only if $l(a)>0$ , $a\in S_{min}^{\tau_n}$ if and only if $l(a)<0$ and $a\in S_{tie}^{\tau_n}$ if and only if $l(a)=0$.
Now, let $S_M$ be a maximizer of the likelihood.
Since $l(a_j)\geq l(a_h) \iff app_w(a_j)\geq app_w(a_h)$ we have that $S_M$, which maximizes $\sum_{a_j \in S} l(a_j)$, is made of top-$k$ alternatives for some $k \in [l \twodots u]$.

Furthermore, $|S_M\cap S_{min}^{\tau_n}| =\max(0,l-k_{tie}^{\tau_n}-k_{max}^{\tau_n})$. Start by noticing that $|S_M\cap S_{min}^{\tau_n}| \geq \max(0,l-k_{tie}^{\tau_n}-k_{max}^{\tau_n})$, since
$|S_M\cap S_{min}^{\tau_n}|\geq l-|S_M\cap S_{max}^{\tau_n}|-|S_M\cap S_{tie}^{\tau_n}|\geq l-k_{max}^{\tau_n}-k_{tie}^{\tau_n} $.
Suppose that $|S_M\cap S_{min}^{\tau_n}| > \max(0,l-k_{tie}^{\tau_n}-k_{max}^{\tau_n})$. Then we have that $|S_M|>l$ because otherwise, if $|S_M|=l$, then $|S_M \cap S_{max}^{\tau_n}|+|S_M \cap S_{tie}^{\tau_n}|=l-|S_M \cap S_{min}^{\tau_n}|<k_{max}^{\tau_n}+k_{tie}^{\tau_n}$, which would mean that there are elements in $S_{tie}^{\tau_n}$ and $S_{max}^{\tau_n}$ which are not in $S_M$, which is a contradiction since $|S_M \cap S_{min}^{\tau_n}|>0$ and $S_M$ is a top-$k$ set. Now consider $a\in S_M \cap S_{min}^{\tau_n}$, we have that $|S_M\backslash\{a\}|\geq l$ and $l(S_M)=l(S_M\backslash\{a\})+l(a)<l(S_M\backslash\{a\})$ which is a contradiction.

With the same idea we can prove that $|S_M\cap S_{max}^{\tau_n}| =\min(u,k_{max}^{\tau_n})$. 

Conversely, consider an admissible set $S$ of top-$k$ alternatives that verifies the constraints (\ref{constraints}). Let $S_M$ be a MLE which, by the first part of the proof, is a top-$k'$ set that also satisfies the same constraints (\ref{constraints}). Thus we have that $|S_M\cap S_{max}^{\tau_n}| =|S\cap S_{max}^{\tau_n}| =\min(u,k_{max}^{\tau_n})$, and since $S$ and $S_M$ are top-$k$ and top-$k'$ sets, we have that $S\cap S_{max}^{\tau_n}=S_M\cap S_{max}^{\tau_n}$. Similarly we have that $S\cap S_{min}^{\tau_n}=S_M\cap S_{min}^{\tau_n}$. This suffices to prove that $l(S)=l(S_M)$ is maximal.
\end{proof}

Notice that when $(l,u)=(0,m)$, the problem degenerates into a collection of label-wise problems, one for each alternative: $a_j$ is selected if $a_j \in S_{max}^{\tau_n}$, rejected if $a_j \in S_{min}^{\tau_n}$, and those that are on the fence can be arbitrarily selected or not.

\begin{example}\label{example constraints}
Consider $5$ alternatives $\mathcal{A}=\{a,b,c,d,e\}$ and $10$ voters $\Voters$ all sharing the same parameters $(p,q) = (0.7,0.4)$. We thus have that all voters share the same weight $w=ln\left(\frac{p(1-q)}{q(1-p)}\right)=1.25$ and $\tau_n=\sum_{i=1}^n ln\left(\frac{1-q}{1-p}\right) = 6.93$. We consider the constraints $(l,u)=(1,4)$

First, suppose that $t_d = 0.6$ and that $t_j =0.5$ for all the remaining candidates. Consider also the approval counts (and weighted approval scores) in the table below.\smallskip

    \begin{tabular}{|c|c|c|c|c|c|}
  \hline
  Candidate & a & b & c & d & e  \\
  \hline
  Approval count &9 & 8 & 7 & 5 & 5  \\
  \hline
  $app_w$ & 11.25 & 10 & 8.75 & 6.65 & 6.25\\
  \hline
\end{tabular}\smallskip

We can easily check, by Theorem \ref{constrained} that $\Tilde{S}=\argmax_{S \in \mathcal{S}} P(S=S^*|A)=\{a,b,c\}$. We have that $S_{max}^{\tau_n}=\{a,b,c\}, S_{tie}^{\tau_n}=\emptyset $ and $S_{min}^{\tau_n}=\{d,e\}$. We know that there exists some $k \in [1,4]$ such that $\tilde{S}$ would consist of the top $k$ alternatives.
We also have that:
\begin{small}
$$
\left\{
    \begin{array}{cl}
        |\tilde{S}\cap S_{max}^{\tau_n}| & =\min(u,k_{max}^{\tau_n})=3  \implies \{a,b,c\} \subseteq \tilde{S} \\
        |\tilde{S}\cap S_{min}^{\tau_n}| & =\max(0,l-k_{tie}^{\tau_n}-k_{max}^{\tau_n})=0  \implies d,e\notin\tilde{S}
        \end{array}
\right. 
$$
\end{small}
So the only possibility is $\tilde{S}=\{a,b,c\}$.
\end{example}

\subsection{Estimating the Parameters Given the Ground Truth}\label{subsec:pa-gt}
\subsubsection{Estimating the prior parameters over alternatives}

Once the ground truths are  estimated at one iteration of the algorithm, the next step consists in estimating the prior parameters $(t_j)_{j \in \mathcal{A}}$, with the ground truths being given (in Subsection \ref{subsec:amle} the ground truth will be replaced by its estimation at each iteration). The next proposition explicits the closed-form expression of the 
MLE of the prior parameter of each alternative given the ground truth of each instance $S^*_z$ once the prior parameters of all other alternatives are fixed.
\begin{compactitem}
    \item Input: Approval profile $(A_1,\dots,A_n)$, ground truths $S^*_z$, and all but one prior parameters  $(t_h)_{h \neq j}$.
    \item Output: MLE of $t_j$.
\end{compactitem}
\begin{proposition}\label{prior}
For every $a_j \in \mathcal{A}$: 
$$\argmax_{t \in (0,1)} \mathcal{L}(A,S,p,q,t,t_{-j}) = \frac{occ(j)\overline{\alpha}_j}{(L-occ(j))\underline{\alpha}_j + occ(j)\overline{\alpha}_j}$$
$$
\mbox{where: }\left\{
    \begin{array}{lll}
        \overline{\alpha}_j &=\sum\limits_{\substack{S\in \mathcal{S}_{l,u} \\ a_j\in S}} \prod\limits_{\substack{a_h \in S \\ h\neq j}} t_h \prod\limits_{a_h \notin S} (1-t_h)\\
        \underline{\alpha}_j & =\sum\limits_{\substack{S\in \mathcal{S}_{l,u} \\ a_j\notin S}} \prod\limits_{a_h \in S} t_h \prod\limits_{\substack{a_h \notin S \\ h\neq j}} (1-t_h) \\
        occ(j) & = \left|z \in \{1,\dots,L\}, a_j \in S_z \right|
        \end{array}
\right.$$
\end{proposition}
Notice that $\overline{\alpha}_j=P(l \leq |S^*| \leq u|a_j \in S^*)$ and $\underline{\alpha}_j=P(l \leq |S^*|\leq u|a_j \notin S^*)$ so $\beta=\overline{\alpha}_jt_j+\underline{\alpha}_j(1-t_j)$. $occ(j)$ is the number of instances whose ground truth contains $a_j$.

\begin{proof}
Fix all sets $S_z \in \mathcal{S}_{l,u}$ and all the noise parameters $(p_i,q_i)_i$ and all the prior parameters $t_h$ but for one $t_j$  for some $j\leq m$, and let $t \in (0,1)$:
\begin{small}
\begin{equation*}
    \begin{split}
      \mathcal{L}(S,t,t_{-j})  & \propto \prod_{z=1}^L \frac{1}{\beta(l,u,t)}\prod_{a_h \in S_z} t_h \prod_{a_h \notin S_z} (1-t_h)\\
       & \propto \prod_{z=1}^L\frac{1}{\beta(l,u,t,t_{-j})}\prod_{a_h \in S_z}t_h \prod_{a_h\notin S_z} (1-t_h)\\
       & \propto \left(\frac{1}{\beta(l,u,t,t_{-j})}\right)^L \underbrace{\prod_{z: a_j \in S_z}t}_{t^{occ(j)}} \underbrace{\prod_{z: a_j \notin S_z} (1-t)}_{(1-t)^{L-occ(j)}}\\
    \end{split}
\end{equation*}
\end{small}
Taking the log we can write the function as:
$$\ell(t)=-L \log \beta + occ(j) \log t + (L-occ(j)) \log (1-t) $$
Its derivative reads:
$$\frac{\partial l}{\partial t} = -L \frac{\underline{\alpha}_j-\overline{\alpha}_j}{\underline{\alpha}_jt+\overline{\alpha}_j (1-t)}+occ(j)\frac{1}{t}+(occ(j)-L)\frac{1}{1-t} $$
Canceling it, we obtain:
$$ {t=\frac{occ(j)\overline{\alpha}_j}{(L-occ(j))\underline{\alpha}_j + occ(j)\overline{\alpha}_j}} $$
The derivative vanishes in a single point in $(0,1)$ and $ \lim_{t\to0} \ell(t)$ $= \lim_{t\to 1}\ell(t)=-\infty $ thus $\ell$ reaches a unique maximum. 
\end{proof}
We will see later that the algorithm applies Proposition \ref{prior} sequentially to estimate the alternatives' parameters one by one (see Example \ref{example amle}).

\subsubsection{Estimating the voter parameters}
Once the ground truths are known (or estimated), we can estimate the voters' parameters $(p,q)$. 
\begin{compactitem}
    \item Input: Instances $(A^1,\dots,A^L)$, ground truths $(S^*_1,\dots,S^*_L)$.
    \item Output: MLE of voter reliabilities $(p,q)$.
\end{compactitem}
The next result simply states that the maximum likelihood estimator of $p_i$ of some voter is the fraction of alternatives that the voter approves and that actually belong to the ground truth; the estimation of $q_i$ is similar. See Example \ref{example amle}. 

\begin{proposition}\label{pq}
Fix sets $S_z \in \mathcal{S}_{l,u}$ and prior parameters $t_j$. Then:
$$\argmax_{(p,q) \in (0,1)^{2 n}} \mathcal{L}(A,S,p,q,t) = (\hat{p},\hat{q})$$
where:
$ \hat{p}_i  = \frac{\sum_{z \in L} |A_i^z\cap {S}_z |}{\sum_{z \in L} |{S}_z |}  
        ,\hat{q}_i  = \frac{\sum_{z \in L} |A_i^z\cap \overline{{S}_z} |}{\sum_{z \in L} |\overline{{S}_z }|}$
\end{proposition}
The (simple) proof is omitted; it is in the Appendix.

\subsection{Alternating Maximum Likelihood Estimation}\label{subsec:amle}

\begin{algorithm}[t]
\caption{\textit{AMLE} procedure 
}
\label{algo}
$\begin{array}{ll}
\textbf{Input:} & \mbox{Approval ballots $(A_i^z)_{1\leq z \leq L , i \in \Voters}$}\\
& \mbox{Initial parameters $\hat{\theta}^{(0)}$}, \mbox{Bounds $(l,u)$}
, \mbox{Tolerance $\varepsilon$}\\
\textbf{Output:} &\mbox{Estimations $(\hat{S}_z), (\hat{p}_i,\hat{q}_i), (\hat{t}_j)$} 
\end{array}$

\begin{algorithmic} 
\REPEAT
\FOR{ $z = 1 \dots L$}
    \STATE Compute $\hat{S}_z^{(v+1)} = \{a_1, \dots, a_k\}$ with $k\in [l,u]$ and:
    $$
\left\{
    \begin{array}{cl}
        |\hat{S}_z^{(v+1)}\cap S_{max,z}^{(v)}| & =\min(u,k_{max,z}^{(v)}) \\
        |\hat{S}_z^{(v+1)}\cap S_{min,z}^{(v)}| & =\max(0,l-k_{tie,z}^{(v)}-k_{max,z}^{(v)})
        \end{array}
\right.$$
\ENDFOR
 
\FOR{ $i=1 \dots \Voters$}
    \STATE Update the parameters $(p_i,q_i)$ given $\hat{S}^{(v+1)}$:
    \begin{small}
        $$
        \hat{p}_i^{(v+1)}  = \frac{\sum\limits_{z \in L} |A_i^z\cap \hat{S}_z^{(v+1)} |}{\sum\limits_{z \in L} |\hat{S}_z^{(v+1)} |}  ,
        \hat{q}_i^{(v+1)}  = \frac{\sum\limits_{z \in L} |A_i^z\cap \overline{\hat{S}_z^{(v+1)}} |}{\sum\limits_{z \in L} |\overline{\hat{S}_z^{(v+1)} }|}
    $$
    \end{small}
\ENDFOR
\FOR{$j = 1 \dots m$}
\STATE Update $\hat{t}_j^{(v+1)}$ by:
\begin{small}
$$ \hat{t}_j^{(v+1)} = \frac{occ^{(v+1)}(j) \overline{\alpha}_j^{(v+1)}}{occ^{(v+1)}(j) \overline{\alpha}_j^{(v+1)} + (L - occ^{(v+1)}(j)) \underline{\alpha}_j^{(v+1)}}$$
\end{small}
where : 
$$\left\{
    \begin{array}{ll}
        occ^{(v+1)}(j) & =\sum_{z=1}^L \mathds{1}\{a_j \in \hat{S}_z^{(v+1)}\}  \\
        \overline{\alpha}_j^{(v+1)} & = \beta((l-1)^+,u-1,\hat{t}^{(v+1)}_{<j},\hat{t}^{(v)}_{>j})\\
        \underline{\alpha}_j^{(v+1)} & = \beta(l,u,\hat{t}^{(v+1)}_{<j},\hat{t}^{(v)}_{>j})
    \end{array}
\right.$$
\ENDFOR
\UNTIL{
$|| \hat{\theta}^{(v+1)}-\hat{\theta}^{(v)} || \leq \varepsilon$}
\end{algorithmic}
\end{algorithm}
Now the estimation of the ground truths and that of the parameters are intertwined to maximize the overall likelihood $\mathcal{L}(A,S,p,q,t)$ by the \emph{Alternating Maximum Likelihood Estimation algorithm}. AMLE is an iterative procedure similar to the \textit{Expectation-Maximization} procedure introduced in \cite{Distill2011} but with a coordinate-steepest-ascent-like iteration, whose aim is to intertwinedly estimate the voter reliabilities, the alternatives' prior parameters and the instances' ground truths. The idea behind this estimation 
consists in alternating a MLE of the ground truths given the current estimate of the parameters, and an updating of these parameters via a MLE based on the current estimate of the ground truths.\footnote{In case of ties between subsets when estimating the ground truth, a tie-breaking 
priority over subsets is used. 
No ties occurred in our experiments.}
Each of these steps have been discussed in the previous subsections and are now incorporated into Algo.~\ref{algo}.

The algorithm continues to run until a convergence criterion is met in the form of a bound on the norm of the change in the parameters' estimations. 
In practice we chose $\ell_\infty$, but any other norm could be used in Algorithm
~\ref{algo} as in finite dimensions, all norms are equivalent (if a sequence converges according to one norm then it does so for any norm).


We define the vector of parameters $\hat{\theta}^{(v)} = (\hat{p}^{(v)},\hat{q}^{(v)},\hat{t}^{(v)})$ 
containing the voters' estimated noise parameters as well as the prior information estimated parameters at iteration $v$.
In particular $ \hat{\theta}^{(0)}$ is the input initial values. 
The choice of the exact initial values 
depends on the application at hand.

Note that at convergence, only local optimality is guaranteed, as classical in optimization.


\begin{theorem}\label{amle conv}
For any initial values $\hat{\theta}^{(0)}$,
\textit{AMLE} converges to a fixed point after a finite number of iterations.
\end{theorem}

\begin{proof}
First we have by Theorem \ref{constrained}
that $ \mathcal{L}(A,\hat{S}^{(v+1)},\hat{\theta}^{(v)}) = \max_{S \in \mathcal{S}} \mathcal{L}(A,S,\hat{\theta}^{(v)})$,
and we have in particular that: 
$$\mathcal{L}(A,\hat{S}^{(v+1)},\hat{\theta}^{(v)})  \geq \mathcal{L}(A,\hat{S}^{(v)},\hat{\theta}^{(v)}) $$

To prove that $\mathcal{L}(A,\hat{S}^{(v+1)},\hat{\theta}^{(v+1)})  \geq \mathcal{L}(A,\hat{S}^{(v+1)},\hat{\theta}^{(v)})$ we use the fact that we update $(p,q,t)$ by their MLE. By Proposition \ref{pq} we have that
$(\hat{p}^{(v+1)},\hat{q}^{(v+1)}) = \argmax_{(p,q)} \mathcal{L}(A,\hat{S}^{(v+1)},p,q,\hat{t}^{(v)})$.
Also by Proposition \ref{prior}, and since we apply it sequentially to update $t_j$ we have:
$$ \mathcal{L}(A,\hat{S}^{(v+1)},\hat{\theta}^{(v+1)})  \geq \mathcal{L}(A,\hat{S}^{(v+1)},\hat{\theta}^{(v)})$$

To prove convergence, it suffices to show that $\hat{S}^{(v)}=\hat{S}^{(v+1)}$ for some $v$ (which guarantees the estimators staying unchanged hereafter). 
Notice that the ground truth has a finite number of possible values (exactly $2^{m L}$), leading
the algorithm to cycle at some iteration. For the sake of simplicity, suppose that this cycle is of length $2$, in other words, suppose that $\hat{S}^{(v+2)}=\hat{S}^{(v)}$ for some $v$; this also implies that $\hat{\theta}^{(v+2)}=\hat{\theta}^{(v)}$. So:
$$    \mathcal{L}(A,\hat{S}^{(v)},\hat{\theta}^{(v)})  = \mathcal{L}(A,\hat{S}^{(v+2)},\hat{\theta}^{(v+2)}) \geq \mathcal{L}(A,\hat{S}^{(v+1)},\hat{\theta}^{(v)})
$$
By optimality of $\hat{S}^{(v+1)}$, we have also that: 
$$ \mathcal{L}(A,\hat{S}^{(v+1)},\hat{\theta}^{(v)}) \geq \mathcal{L}(A,\hat{S}^{(v)},\hat{\theta}^{(v)})$$
Hence, we get that: $$\mathcal{L}(A,\hat{S}^{(v+1)},\hat{\theta}^{(v)}) = \mathcal{L}(A,\hat{S}^{(v)},\hat{\theta}^{(v)})$$
and thus,
$\hat{S}^{(v+1)}=\hat{S}^{(v)}=\argmax_{S \in \mathcal{S}_{l,u}} \mathcal{L}(A,S,\hat{\theta}^{(v)})$
and the estimators will remain the same after any number of iterations following $v$.
\end{proof}
Because $\mathcal{L}(A,\hat{S}^{(v+1)},\hat{\theta}^{(v+1)})  \geq \mathcal{L}(A,\hat{S}^{(v+1)},\hat{\theta}^{(v)}) \geq \mathcal{L}(A,\hat{S}^{(v)},\hat{\theta}^{(v)})$, the likelihood increases at each step of the algorithm. This guarantees that whenever the execution stops, the likelihood is closer to the maximum than it initially was. Therefore the algorithm can not only be run until convergence, but it can also be run as an anytime algorithm. 

\begin{example}\label{example amle}
Take $n = 3$, $m = 5$, $l=1$, $u=2$, 
$L=4$, and the following profile and initial parameters:
\begin{table}[h]
    \centering
    \begin{tabular}{|l|c|c|c|c|}
  \hline
 & $A^1$  & $A^2$ & $A^3$ & $A^4$ \\
  \hline
  Voter $1$ & $\{a_1,a_4\}$ & $\{a_1\}$ & $\{a_3\}$ & $\{a_1\}$ \\
  \hline
  Voter $2$ & $\{a_2\}$ & $\{a_5\}$ & $\{a_4\}$ & $\{a_1\}$\\
  \hline
  Voter $3$ & $\{a_2,a_3,a_4\}$ & $\{a_2,a_3,a_5\}$ & $\{a_2,a_3\}$ & $\{a_3\}$ \\
  \hline 
\end{tabular}
    \label{approval profile}
\end{table}
$$\left\{
    \begin{array}{lll}
        \hat{p}_1^{(0)} = 0.5 & \hat{p}_2^{(0)} = 0.5 & \hat{p}_3^{(0)} = 0.5  \\
        \hat{q}_1^{(0)} = 0.44 & \hat{q}_2^{(0)}  =0.41 & \hat{q}_3^{(0)}  =0.32 \\
        \hat{t}_1^{(0)}=\dots=\hat{t}_5^{(0)} &  =0.5
    \end{array}
\right.$$

\paragraph{Estimating the ground truth:} The first step 
is the application of Theorem \ref{constrained} to estimate the ground truth of the instances given the initial parameters, yielding
\begin{small}
${\hat{S}_1^{(1)}= \{a_2,a_4\}, \hat{S}_2^{(1)} = \{a_2,a_5\}, \hat{S}_3^{(1)} = \{a_2,a_3\}, \hat{S}_4^{(1)} = \{a_1,a_3\} }
$.
\end{small}
\paragraph{Estimating the voter reliabilities:} In the next step we use these estimates of the ground truths to compute the MLEs of the voter reliabilities.
For instance, voter $1$ has 2 false positive labels from a total of $12$ negative labels so $\hat{q}_1^{(1)}=\frac{2}{12}=0.17$ and she has 3 true positive labels out of 8 positive ones so $\hat{p}_1^{(1)}=\frac{3}{8}=0.38$. In the end, we get:
$$\left\{
    \begin{array}{lll}
        \hat{p}_1^{(1)}=0.38  & \hat{p}_2^{(1)}=0.38 & \hat{p}_3^{(1)}=0.88\\
        \hat{q}_1^{(1)}=0.17  & \hat{q}_2^{(1)}=0.08 & \hat{q}_3^{(1)}=0.17
    \end{array}
\right.$$


\paragraph{Estimating the prior parameters:} The final step of this iteration consists in updating the estimations of the prior parameters by applying Proposition \ref{prior} sequentially.
First we estimate $\hat{t}_1^{(1)}$ given $\hat{S}^{(1)}$ and $\hat{t}_2^{(0)},\dots,\hat{t}_5^{(0)}$ by maximum likelihood estimation. We first compute $\overline{\alpha}_1$, $\underline{\alpha}_1$ and $occ(a_1)$:
$$\left\{
    \begin{array}{ll}
        \overline{\alpha}_1 & = \beta(0,1,t_2,\dots,t_5)=0.3125\\
        \underline{\alpha}_1 & = \beta(1,2,t_2,\dots,t_5)=1\\
        occ(a_1) & = 1
    \end{array}
\right.$$
Then the maximum likelihood estimation of $t_1$ is:
$$\hat{t}_1 = \frac{occ(a_1)\overline{\alpha}_1}{(L-occ(a_1))\underline{\alpha}_1 + occ(a_1)\overline{\alpha}_1}= 0.09  $$
The next steps are to estimate $\hat{t}_2^{(1)}$ given $\hat{t}_1^{(1)},\hat{t}_3^{(0)},\hat{t}_4^{(0)},\hat{t}_5^{(0)}$ and so on. Finally, we get:
$$\hat{t}_1^{(1)} = 0.09,\hat{t}_2^{(1)} = 0.56 , \hat{t}_3^{(1)} = 0.28, \hat{t}_4^{(1)} = 0.14, \hat{t}_5^{(1)} = 0.20 $$

Fix $\varepsilon=10^{-5}$. We repeat all steps until convergence (according to $ \ell_{\infty}$),
after $5$ full iterations. In the fixed point, the estimations of the ground truths are:
$$\hat{S}_1 = \{a_2,a_3\},  \hat{S}_2 = \{a_2,a_3\}, \hat{S}_3 = \{a_2,a_3\}, \hat{S}_4 = \{a_3\}$$
\end{example}

\section{Experiments}\label{sec: experiments}
\subsection{Experiment Design and Data Collection}
We designed an image annotation task as a football quiz.\footnote{The annotations dataset and the code are available at: \url{https://github.com/taharallouche/Football-Quiz-Crowdsourcing}}. We selected $15$ pictures taken during different matches between two of the following teams: Real Madrid, Inter Milan, Bayern Munich, Barcelona, Paris Saint-Germain.
In each picture, it may be the case that players from both teams appear, or players from only one team, therefore 
$l=1$ and $u=2$. Each participant is shown the instances one by one, and is each time asked to select all
the teams she can spot (see Figure \ref{Quiz}). We 
designed a simple incentive for participants, consisting in ranking them according to the following principle:
\begin{compactitem}
    \item The participants get one point whenever their answer contains all correct alternatives for a picture. They are then ranked according to their cumulated points. 
    \item To break ties, the participant who selected a smaller number of alternatives overall is ranked first.
\end{compactitem}

We gathered the answers of $76$ participants \bt(only two of them spammed by simply selecting all the alternatives).\et
\begin{figure}[h]
     \centering
         \includegraphics[width=0.40\textwidth]{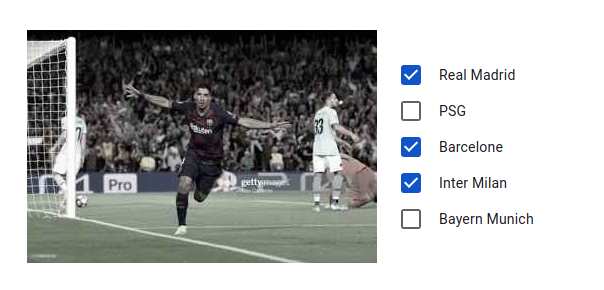}
    \caption{Example of Annotation Task}
    \label{Quiz}
\end{figure}

\subsection{Anna Karenina's Initialization}

Inspired by the \emph{Anna Karenina Principle} in~\cite{truth2019}, we assign more weight to voters who are \emph{closer} to the others on average, initializing the precision parameters $(p_i,q_i)$ accordingly. 
This suits our context, \bt where voter competence is highly polarized: some voters are experts and cast similar answers close to the ground truth, the others are less reliable and their answers are dispersed among all combinations. Details on the heuristic are 
in the Appendix.

\subsection{Results}

To assess the importance of 
prior information on the size of the ground truth, we tested the AMLE algorithm with free bounds $(l,u)=(0,m)$ (will be referred to as $\mbox{AMLE}_f$) and the $\mbox{AMLE}_c$ algorithm with $(l,u)=(1,2)$. 
We also apply the modal rule \cite{Evaluating2020} which 
outputs the subset of alternatives that most frequently appears as an approval ballot $\argmax_{S \in \mathcal{S}} 
\left|i \in \Voters, S=A_i \right|
$,  and a variant of label-wise majority rule which 
outputs the subset of alternatives $S$ such that
$a \in S \iff \left|i \in \Voters, a \in A_i \right| > \frac{n}{2}$. If this subset is empty it is replaced by the alternative with highest approval count, and if it has more than two alternatives then we only keep the top-2 alternatives. 


We took $20$ batches of $n=10$ to $n=74$ randomly drawn voters and applied the four methods to all of them (see Figure \ref{foot_voters_ham},\ref{foot_voters_01}). As classically done in the literature~\cite{aggregation2020}, we use the Hamming accuracy $\frac{1}{m L} \sum_{z=1}^L |S^*_z \cap \hat{S}^z|+|\overline{S^*_z} \cap \overline{\hat{S}^z}|  $
and the 0/1 accuracy
$\frac{1}{L} \sum_{z=1}^L \mathds{1}\{S^*_z = \hat{S}^z\} $  as metrics and report their 0.95 confidence intervals.

\begin{figure}
\centering
\begin{minipage}[b]{0.9\linewidth}
\centering
         \includegraphics[width=\textwidth]{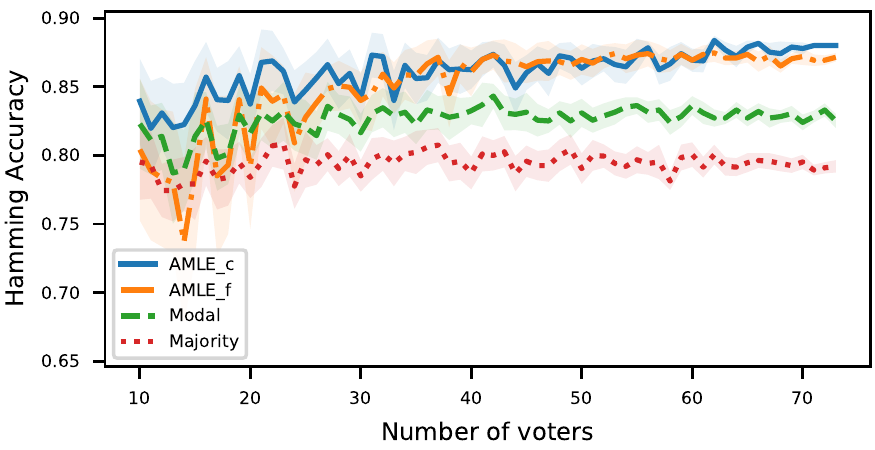}
             \subcaption{Hamming accuracy}
        \label{foot_voters_ham}
\end{minipage}
\begin{minipage}[b]{0.9\linewidth}
\centering
         \includegraphics[width=\textwidth]{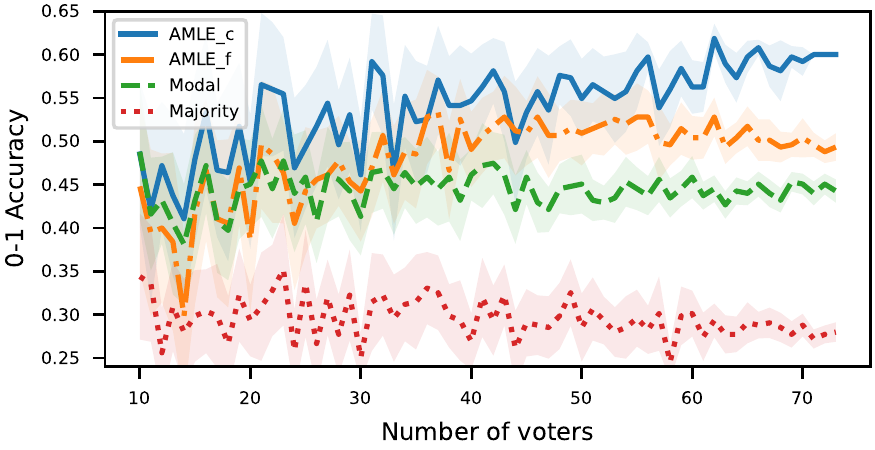}
        \subcaption{0/1 accuracy}
        \label{foot_voters_01}
\end{minipage}
\caption{Accuracies of different aggregation methods}
        \label{foot_voters}
\end{figure}

We notice 
that the majority and the modal rule are outperformed by AMLE, which can be explained by the fact that they do not take into account the voters' reliabilities. 
Comparing the performances of $\mbox{AMLE}_c$ and $\mbox{AMLE}_f$ emphasizes the importance of the prior knowledge on the committee size to improve the quality of the estimation.

\section{Conclusion}\label{conclusion}
We study multi-winner approval voting from an epistemic point of view. We propose a noise model that incorporates the prior belief about the size of the ground truth. Then we derive an iterative algorithm to intertwinedly estimate the ground truth labels, the voter noise parameters and the prior belief parameters and we prove its convergence. Our algorithm is based on a simplification of Expectation-Maximization (EM), and its simple steps are more easily explainable to voters than EM and other similar statistical learning approaches.

Although we mainly considered a general multi-instance task that fits the collective annotation framework, where each voter answers several questions on the same set of alternatives, we can nonetheless apply the same algorithm to single-instance problems (such as the allocation of scarce medical resources) where only one question is answered. In this case, the prior parameters cannot be updated and it suffices to fix them once and for all and alternate between the estimation of the ground truth and the voter parameters.

In some contexts ({\em e.g.}, patients in a hospital), alternatives and votes are not observed at once but streamed. To cope with this online setup we consider extending our AMLE algorithm in the spirit of \cite{online2009}.

\bibliographystyle{named}
\bibliography{ms.bib}

\newpage
\begin{huge}
\noindent \textbf{Appendix}
\end{huge}
\section*{Data collection and incentives}
To see how the participants behave given the ranking incentives that we defined in the football quiz, we plotted the histogram of the sizes of the answers (see Figure \ref{hist}). It appears that although the platform enables to select every alternative, only two voters did so for all the questions. Moreover, figures $\ref{sw_hist}$ and $\ref{two_hist}$ show that the majority of the voters tend to select exactly the number of teams that appear in an image. 

\begin{figure}[h]
     \centering
     \begin{subfigure}[b]{0.49\textwidth}
          \centering
         \includegraphics[width=0.85\textwidth]{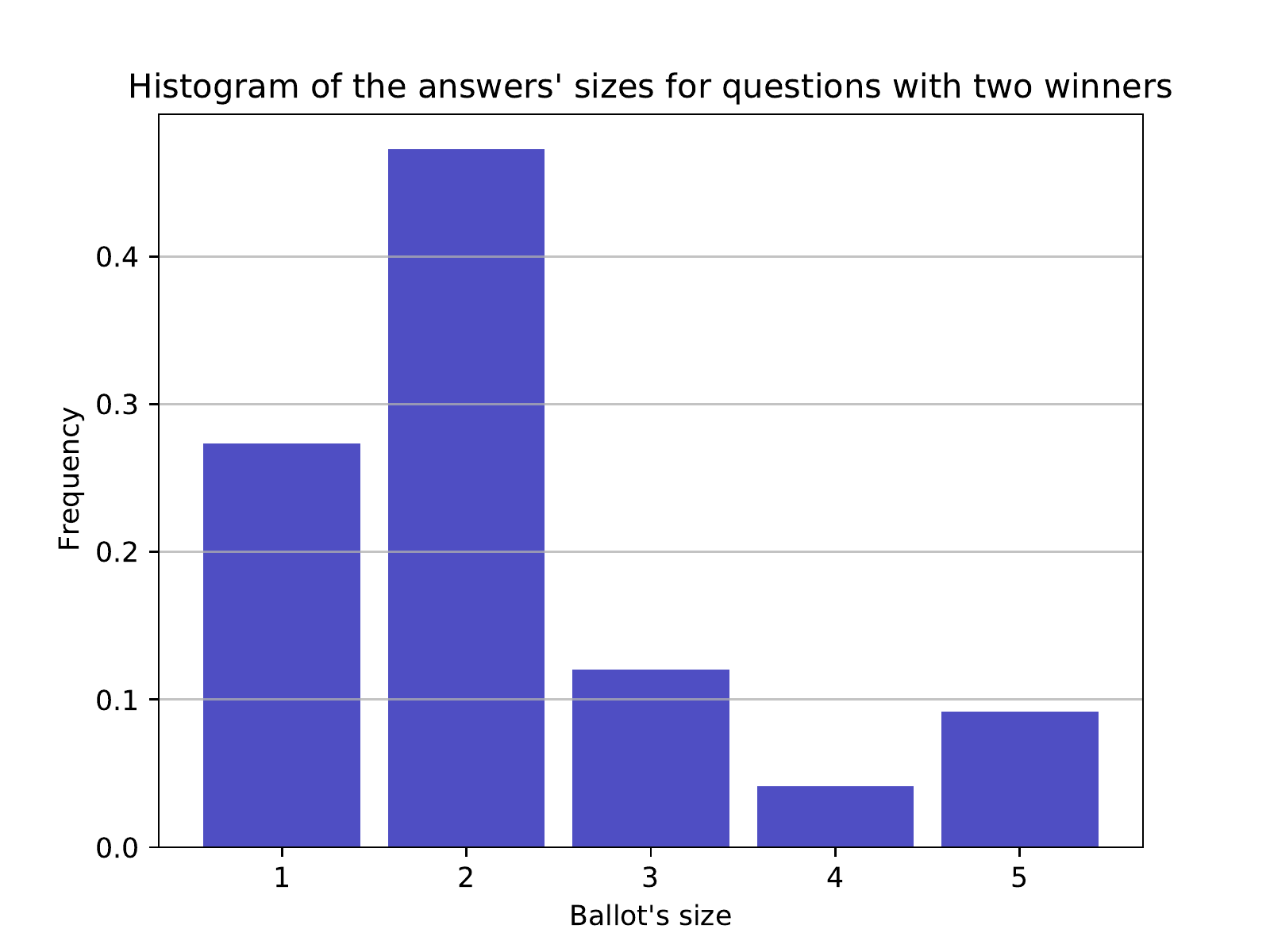}
        \subcaption{Two-winner instances}
        \label{two_hist}
     \end{subfigure}
     \begin{subfigure}[b]{0.49\textwidth}
          \centering
         \includegraphics[width=0.85\textwidth]{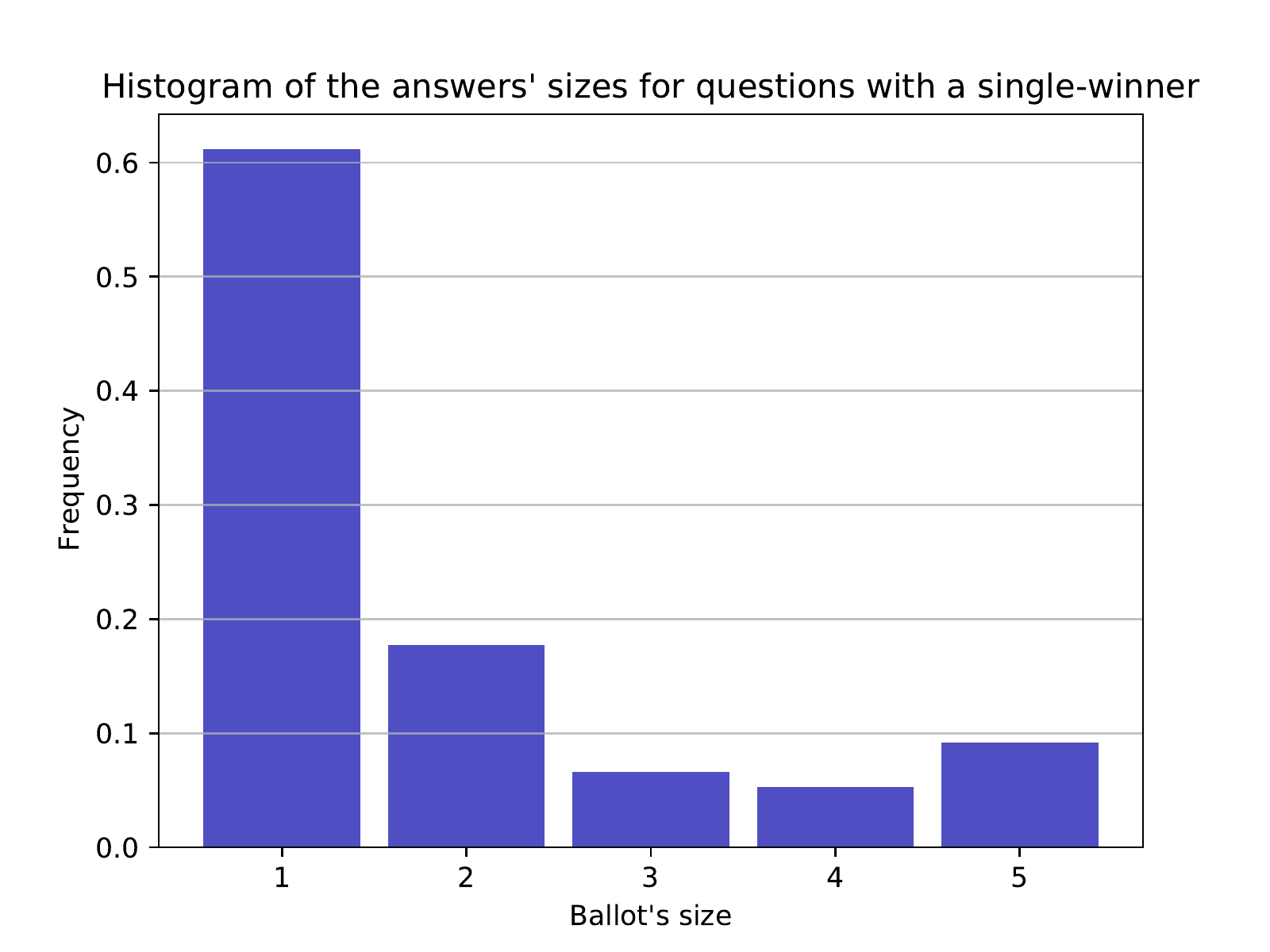}
        \subcaption{Single-winner instances}
        \label{sw_hist}
     \end{subfigure}
        \caption{Histogram of the ballots' sizes}
        \label{hist}
\end{figure}

\section*{Initializing Voters' Reliabilities}
Inspired by the \emph{Anna Karenina Principle} in~\cite{truth2019}, we devised an initialisation strategy for the voters' reliabilities. In his book, Leo Tolstoi stated that "Happy families are all alike; every unhappy family is unhappy in its own way". In the same spirit, it seems reasonable to make the hypothesis that accurate users tend to make similar answers, whereas inaccurate users have each their own way of being inaccurate.

 We use the following heuristic (see Algorithm \ref{algo_init}) for the initialization. We used the Jaccard distance given by:
 $$d_{Jacc}(A,B)= \frac{|\overline{A}\cap B|+|A \cap \overline{B}|}{|A \cup B|} $$
\begin{algorithm}[H]
\caption{ Initializing $(p_i,q_i)_i$ }
\label{algo_init}
$\begin{array}{ll}
\textbf{Input:} & \mbox{Approval ballots $(A_i^z)_{z,i}$}\\
\textbf{Output:} &\mbox{Initialization $(\hat{p}^{(0)}_i,\hat{q}^{(0)}_i)$} 
\end{array}$
\begin{algorithmic} 
\STATE -Compute $w_{max}=\frac{n}{1+n}, w_{min}=\frac{1}{1+n}$
\STATE -Compute $d_i = \sum_{j\neq i} d_{Jacc}(A_i,A_j)$
\STATE -Compute $d_{max} = \max d_i, d_{min} = \min d_i $
\STATE -Compute $w_i = (w_{max}-w_{min})\left(\frac{\frac{1}{d_i}-\frac{1}{d_{max}}}{\frac{1}{d_{min}}-\frac{1}{d_{max}}}\right)+w_{min}$
\STATE -Fix $\hat{p}^{(0)}_i= \frac{1}{2}$ and $\hat{q}^{(0)}_i=\frac{1-\frac{e^{w_i}-1}{e^{w_i}+1}}{2}$
\end{algorithmic}
\end{algorithm}

\begin{remark}
The formulas in Algorithm \ref{init} guarantee that a voter's parameters $(\hat{p}^{(0)}_i,\hat{q}^{(0)}_i)$ are such that her initial weight is equal to $w_i$, and that $\frac{w_{max}}{w_{min}}=n$ which means that initially, the voter closest in average to the other voters counts $n$ times the voter with biggest average distance. 
\end{remark}
\begin{example}
Consider following the approval profile (Table \ref{app profile}) for $3$ voters, $5$ alternatives and $4$ Instances.
\begin{table}[h]
    \centering
    \begin{tabular}{|l|c|c|c|c|}
  \hline
 & $A^1$  & $A^2$ & $A^3$ & $A^4$ \\
  \hline
  Voter $1$ & $\{a_1,a_4\}$ & $\{a_1\}$ & $\{a_3\}$ & $\{a_1\}$ \\
  \hline
  Voter $2$ & $\{a_2\}$ & $\{a_5\}$ & $\{a_4\}$ & $\{a_1\}$\\
  \hline
  Voter $3$ & $\{a_2,a_3,a_4\}$ & $\{a_2,a_3,a_5\}$ & $\{a_2,a_3\}$ & $\{a_3\}$ \\
  \hline 
\end{tabular}
    \caption{Approval Ballots of 3 Voters on 4 Instances}
    \label{app profile}
\end{table}
Here we have that:
$$w_{max}=\frac{n}{n+1}=0.75, w_{min}=\frac{1}{n+1}=0.25 $$
First, compute the mean Jaccard distance of all voters:
$d_1=1.71, d_2=1.69,d_3=1.65 $.
So $d_{max}=d_1=1.71$ and $d_{min}=d_3=1.65$, which means that voter $3$ (the closest in average to all the voters) will get the biggest weight $w_3=w_{max}=0.75$ and voter $1$ gets the smallest weight $w_1=w_{min}$.
Next, compute the weight that will be assigned to each voter, for instance:
$$w_2=(w_{max}-w_{min})\frac{\frac{1}{d_2}-\frac{1}{d_{max}}}{\frac{1}{d_{min}}-\frac{1}{d_{max}}}+w_{min}=0.38 $$
Now we can set the initial values for the reliability parameters accordingly:
$$\hat{p}^{(0)}_2= \frac{1}{2} ,\hat{q}^{(0)}_2=\frac{1-\frac{e^{w_2}-1}{e^{w_2}+1}}{2} $$
We can check that these parameters are such that:
$$ln\left[ \frac{p_2(1-q_2)}{q_2(1-p_2)}\right]=w_2 $$
After proceeding in the same fashion with all the voters, we get the initial parameters:
$$\left\{
    \begin{array}{lll}
        \hat{p}_1^{(0)} = 0.5 & \hat{p}_2^{(0)} = 0.5 & \hat{p}_3^{(0)} = 0.5  \\
        \hat{q}_1^{(0)} = 0.44 & \hat{q}_2^{(0)}  =0.41 & \hat{q}_3^{(0)}  =0.32 \\
    \end{array}
\right.$$

\end{example}

Since the AMLE only guarantees convergence to a local maximum, which makes the result depending on the initial point, we compared the results of this initialization (Anna Karenina) to other procedures to motivate its choice, see Figure \ref{init}, namely we tested:
\begin{itemize}
    \item Uniform weights: Initially all the voters in the batch are given the same weight.
    \item Random weights: Initially, for each voter in the batch, $p_i$ is randomly picked from $(0.5,1)$ and $q_i$ is randomly picked from $(0,0.5)$.
\end{itemize}
We can notice that these two baseline procedures show very similar performances, and that they are both outperformed by the Anna Karenina initialization.

\begin{figure}[h]
     \centering
     \begin{subfigure}[b]{0.45\textwidth}
         \centering
         \includegraphics[width=0.85\textwidth]{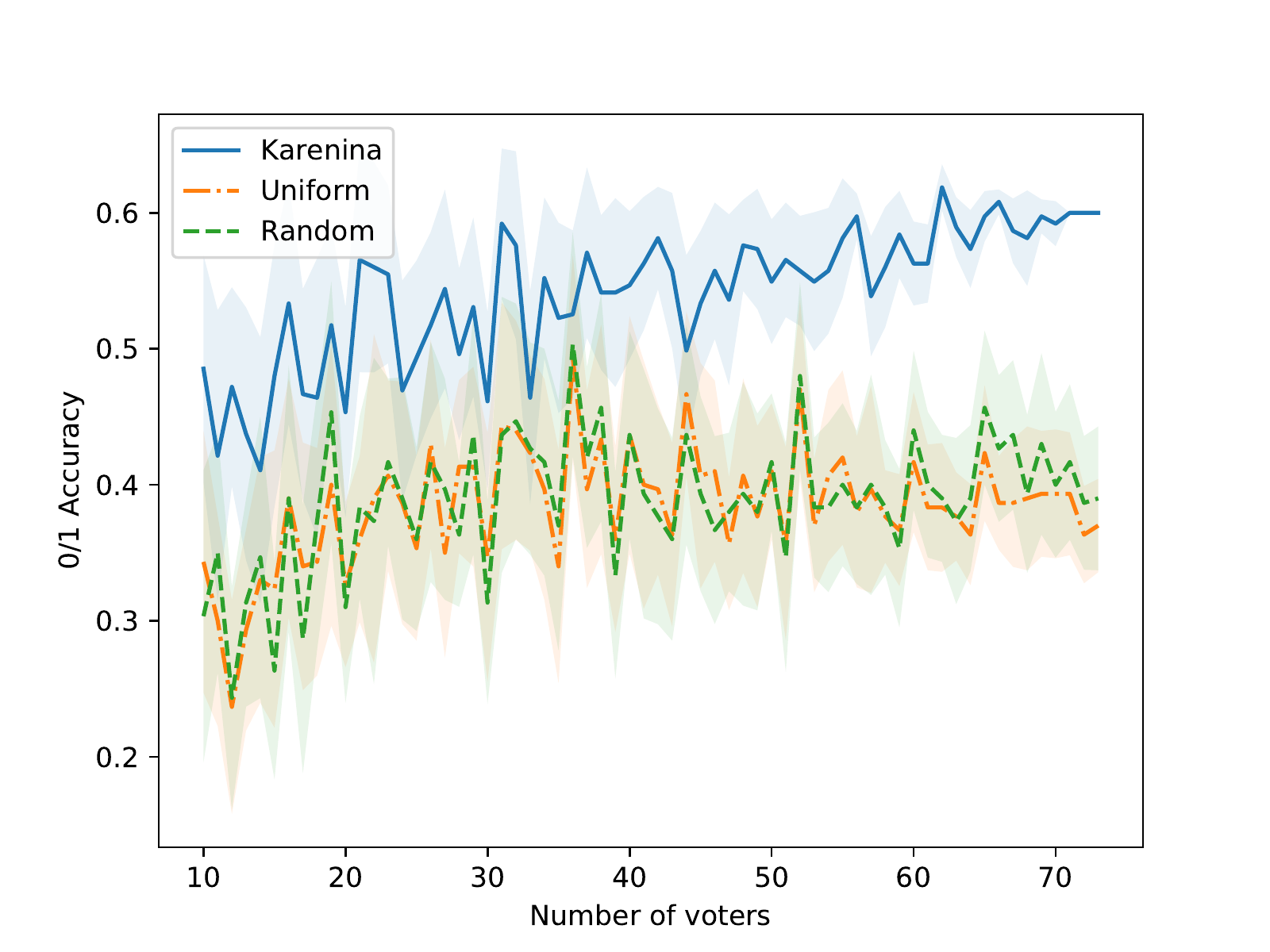}
             \subcaption{0-1 accuracy}
        \label{init_01}
     \end{subfigure}
     \begin{subfigure}[b]{0.45\textwidth}
          \centering
         \includegraphics[width=0.85\textwidth]{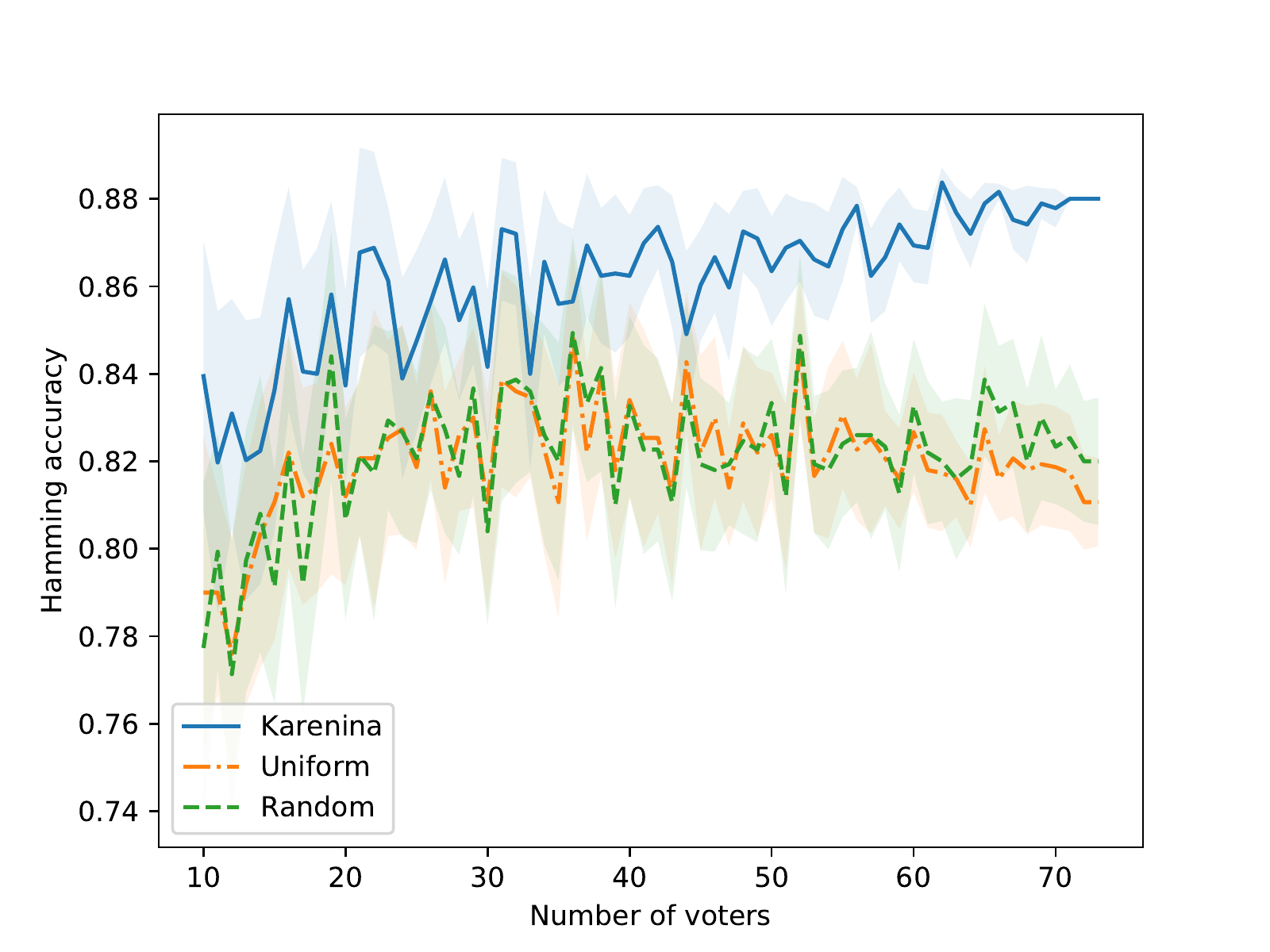}
        \subcaption{Hamming accuracy}
        \label{init_ham}
     \end{subfigure}
        \caption{Accuracies of different initializations}
        \label{init}
\end{figure}

\section*{Time Complexity of AMLE}
We assessed the execution time of the AMLE algorithm with and without constraints (refered to as AMLE and AMLE$_f$), run on Intel Core i7-10610U CPU @1.80Ghz 4 cores, 8 threads and 32Gb RAM. Results are show in Figure \ref{complexity}. We can see that whereas the number of iteration does not seem to grow as the number of voter increases, the execution time of AMLE does, especially around $40$ voters.
\begin{figure}[H]
     \centering
     \begin{subfigure}[b]{0.45\textwidth}
         \centering
         \includegraphics[width=0.85\textwidth]{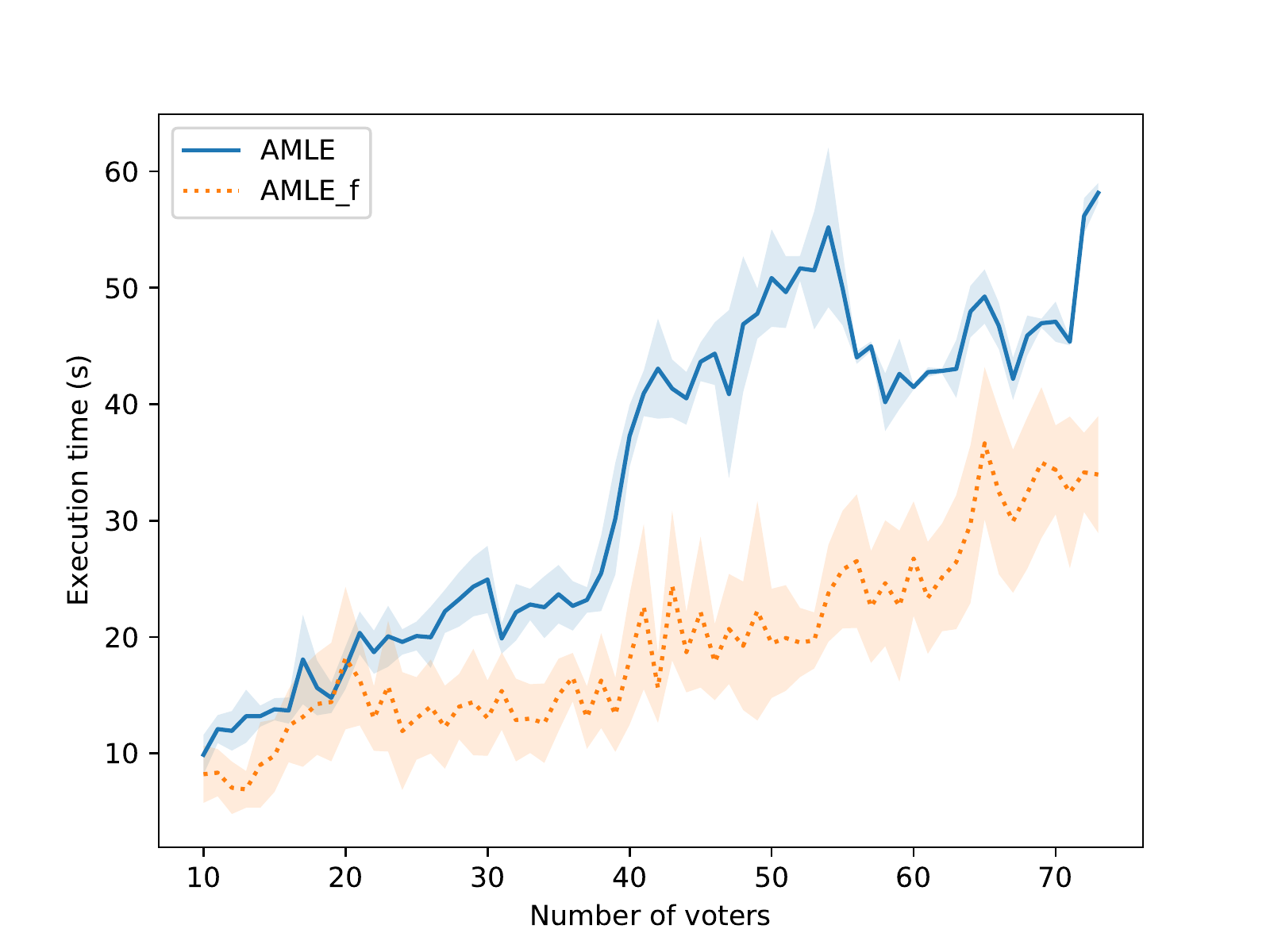}
             \subcaption{Execution time of AMLE}
        \label{time}
     \end{subfigure}
     \begin{subfigure}[b]{0.45\textwidth}
          \centering
         \includegraphics[width=0.85\textwidth]{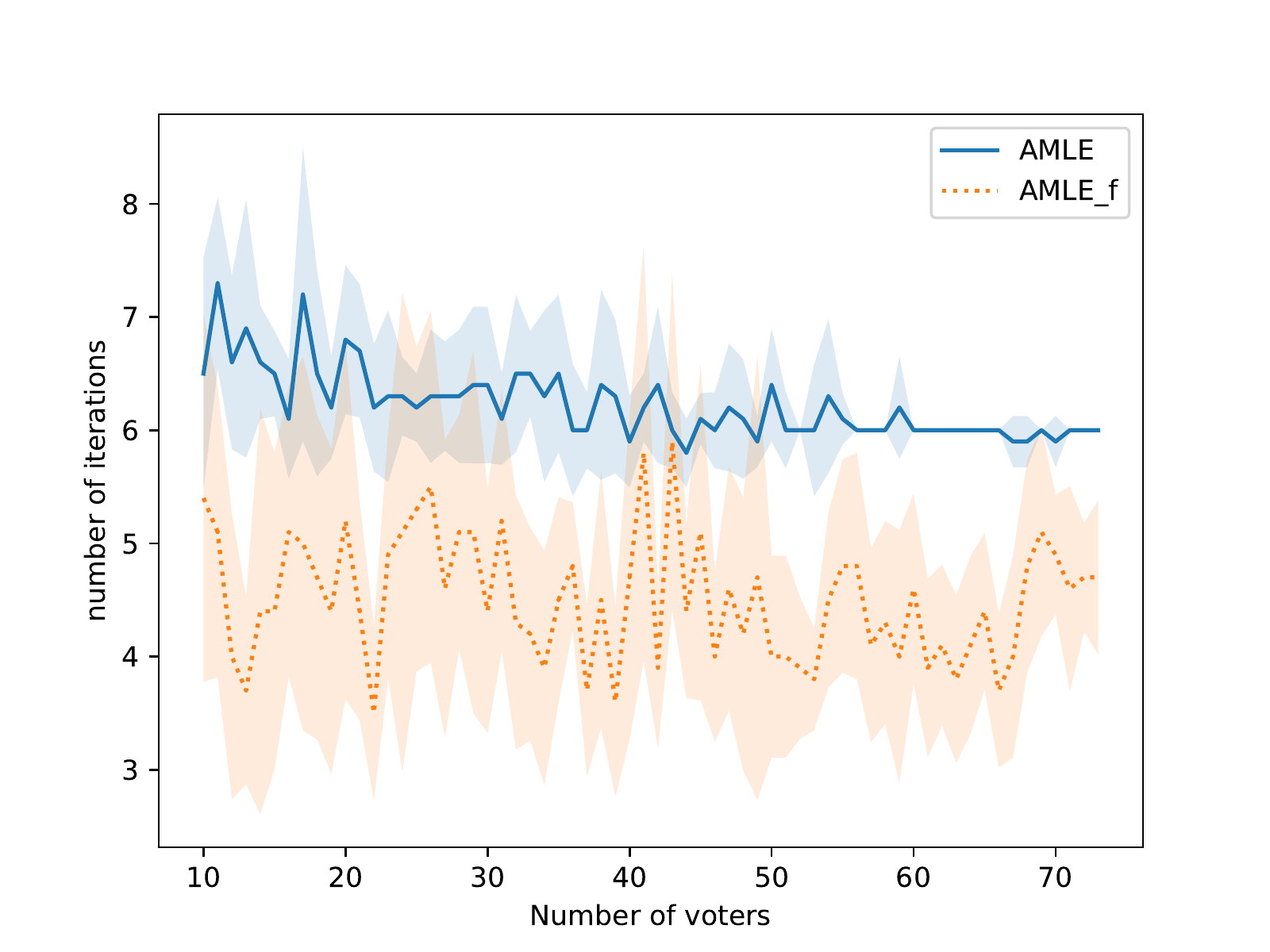}
        \subcaption{Number of iterations of AMLE}
        \label{iterations}
     \end{subfigure}
        \caption{Time complexity of AMLE}
        \label{complexity}
\end{figure}

\section*{Losses}
\subsection*{Hamming, Harmonic and 0-1 Subset Metrics}
In addition to the Hamming and 0-1 subset accuracies, we introduced a new metric which can be considered as an intermediate one. The Hamming metric considers each label independently and the 0-1 subset loss considers them jointly in a strict fashion, whereas the harmonic accuracies that we introduced considers all the instance's labels jointly but with different convex weights depending on the number of correctly predicted ones:
$$T(S,S^*) = \sum_{k=1}^{|S\cap S^*|} \frac{1}{6-k} $$
So out of the 5 labels:
\begin{itemize}
    \item if 0 labels are correct then $T = 0$.
    \item if 1 labels is correct then $T = \frac{1}{5}$.
    \item if 2 labels are correct then $T = \frac{1}{5}+\frac{1}{4}$.
    \item if 3 labels are correct then $T = \frac{1}{5}+\frac{1}{4}+\frac{1}{3}$.
    \item if 4 labels are correct then $T = \frac{1}{5}+\frac{1}{4}+\frac{1}{3}+\frac{1}{2}$.
    \item if 5 labels are correct then $T = \frac{1}{5}+\frac{1}{4}+\frac{1}{3}+\frac{1}{2}+1$.
\end{itemize}

Defined as such, this accuracy favours the estimators that are able to correctly estimate most of the instance's labels without being as rigid as the 0-1 subset accuracy.
 
This metric is reminiscent of the Proportional Approval Voting rule for multiwinner elections, which defines the score of a subset of candidates $W$ for a voter as $1 + \frac12 + \ldots + \frac1j$, where $j$ is the number of candidates in $W$ approved by the voter. We could consider more generally a class of metrics defined by a vector $\vec{w}$, such that $T(S,S^*) = w_{|S \cap S^*|}$. This class generalizes Hamming, 0-1 and Harmonic and is reminiscent of the class of {\em Thiele} rules (see for instance \cite{LacknerS20} for an extended presentation of multiwinner approval-based committee rules).

\subsection*{Results}
We show in Table \ref{entire_dataset} the accuracies of the considered methods when applied to the entire annotation dataset. In Figure \ref{harmonic} we show the evolution of the Harmonic accuracies when the number of randomly picked voters in each batch increase.
\begin{table}[h]
    \centering
    \begin{tabular}{|l|c|c|c|c|}
  \hline
   &$\mbox{AMLE}_c$  & $\mbox{AMLE}_f$ & Modal & Majority \\
  \hline
  Hamming & \textbf{0.88} & 0.86 & 0.84 & 0.80 \\
  \hline
  Harmonic & \textbf{0.78} & 0.74  & 0.69 & 0.61\\
  \hline
  0/1 & \textbf{0.60} & 0.53  & 0.46 & 0.26\\
  \hline
\end{tabular}
    \caption{Hamming and 0/1 accuracy for entire dataset}
    \label{entire_dataset}
\end{table}

\begin{figure}[h]
         \centering
         \includegraphics[width=0.4\textwidth]{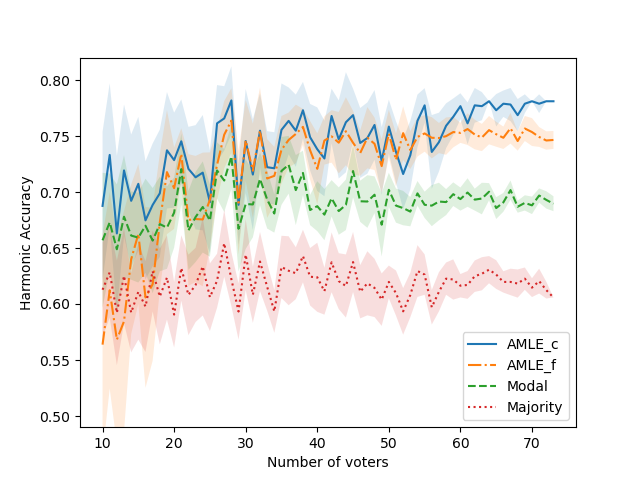}
         \caption{Normalized Harmonic accuracy}
        \label{harmonic}
\end{figure}

\end{document}